\pdfoutput=1

\documentclass[11pt]{article}

\usepackage[preprint]{coling}

\usepackage{times}
\usepackage{latexsym}
\usepackage{xspace}
\usepackage{amsmath}
\usepackage{amsthm}
\usepackage{subfig}
\usepackage{booktabs}

\usepackage[T1]{fontenc}
\usepackage{multirow}

\usepackage[utf8]{inputenc}

\usepackage{microtype}

\usepackage{inconsolata}

\usepackage{graphicx}

\newcommand{\ie}{{\emph{i.e.,}}\xspace}
\newcommand{\eg}{\emph{e.g.,}\xspace}

\newtheorem{proposition}{Proposition}

\newtheorem{theorem}{Theorem}
\newcommand{\stitle}[1]{\vspace{1.6ex}\noindent{\bf #1}}
\newcommand{\eat}[1]{{}}

\newcommand{\figref}[1]{{Fig.~\ref{#1}}}
\newcommand{\secref}[1]{{Sec.~\ref{#1}}}
\newcommand{\equref}[1]{{Equ.~\ref{#1}}}
\newcommand{\tabref}[1]{{Table~\ref{#1}}}

\newcommand{\propref}[1]{{Proposition~\ref{#1}}}
\newcommand{\thmref}[1]{{Theorem~\ref{#1}}}

\newcommand{\ourmodel}{AsymKV}
\newcommand{\mathattn}{\mathcal{M}}
\newcommand{\mathsm}{sm}
\newcommand{\mathcat}{cat}

\title{AsymKV: Enabling 1-Bit Quantization of KV Cache with Layer-Wise Asymmetric Quantization Configurations}

\author{
 \textbf{Qian Tao\textsuperscript{1}},
 \textbf{Wenyuan Yu\textsuperscript{1}},
 \textbf{Jingren Zhou\textsuperscript{2}}
\\
 \textsuperscript{1}Tongyi Lab, Alibaba Group
 \textsuperscript{2}Alibaba Cloud Computing, Alibaba Group
\\
 \small{
   \textbf{Correspondence:} \href{mailto:qian.tao@alibaba-inc.com}{qian.tao@alibaba-inc.com}
 }
}

\begin{document}
\maketitle
\begin{abstract}
Large language models have shown exceptional capabilities in a wide range of tasks, such as text generation and video generation, among others. 
However, due to their massive parameter count, these models often require substantial storage space, imposing significant constraints on the machines deploying LLMs.
To overcome this limitation, one research direction proposes to compress the models using integer replacements for floating-point numbers, in a process known as Quantization.
Some recent studies suggest quantizing the key and value cache (KV Cache) of LLMs, and designing quantization techniques that treat the key and value matrices equivalently.

This work delves deeper into the asymmetric structural roles of KV Cache, a phenomenon where the transformer's output loss is more sensitive to the quantization of key matrices.
We conduct a systematic examination of the attention output error resulting from key and value quantization.
The phenomenon inspires us to propose an asymmetric quantization strategy.
Our approach allows for 1-bit quantization of the KV cache by implementing distinct configurations for key and value matrices.
We carry out experiments across a variety of datasets, demonstrating that our proposed model allows for the quantization of up to 75\% decoder layers with 1 bit, while simultaneously maintaining performance levels comparable to those of the models with floating parameters.
\end{abstract}

\section{Introduction}
\label{sec:intro}

Large language models (LLMs) have gained considerable interest of late due to their remarkable performance in various directions~\cite{textgenbook,engineering22machine,arxiv22gala,arxiv21mental,nips24timeseries}.
However, to achieve a high level of expressiveness, LLMs typically require billions of parameters, which necessitates substantial storage space and poses challenges for deployment on machines with limited resources.

A line of research has been dedicated to enabling the deployment of these models on machines with less available space through model compression techniques.
One such technique, model quantization, aims to represent the parameter matrices in LLMs using fewer bits (\eg integer, binary), thereby making them more suitable for deployment on hardware with limited storage capacity~\cite{arxiv23squeeze}.
More recently, the Key-Value cache (KV cache) in LLMs has been shown to occupy a large proportion of space~\cite{mlsys23efficient,arxiv23landmark}, especially when the length of context increases, and numerous works have focused on the quantization of KV cache~\cite{arxiv24intact,icml24kivi,arxiv24gear}.
Nonetheless, these studies typically employ the same quantization configuration for both key and value matrices.

In this paper, we cast a spotlight on the asymmetric structural roles of key and value matrices.
Our analysis reveals that, while a quantization method could yield a quantized matrix with a commensurate loss for both key and value matrices, \emph{the multiplication of query and application of the activation function to the key matrix} results in a larger loss of key matrix in the transformer's output as compared to the value matrix.

Drawing on this observation, this paper introduces a simple yet efficacious quantization strategy, which entails the use of asymmetric and layer-wise quantization configurations for key and value matrices.
Specifically, during the next token's inference, we employ a higher-bit quantization strategy (for instance, a 4-bit strategy) for the first $l$ decoder layers, whilst a lower-bit strategy (\ie the 1-bit strategy) is applied for the remaining decoder layers.
For key and value matrices, we choose different $l$ to account for their asymmetric structural positions.
Our extensive experiments reveal that the adoption of an asymmetric and layer-wise quantization strategy allows us to quantize a subset of layers using a $1$-bit approach, resulting in a strategy that is both space and computationally efficient.

In summary, the primary contributions of this paper can be outlined as follows:
\begin{itemize}
    \item We conduct the exploration of the asymmetric structural roles of the key and value matrices. Through practical and theoretical demonstrations, we show that the loss derived from the key matrix's quantization will be magnified relative to that of the value matrix, owing to the multiplication of the query and activation function applied specifically to the key matrix.
    \item To counteract the impact of asymmetric structural roles, this paper proposes {\ourmodel}, a simple yet effective approach that combines varied degrees of quantization configurations at the layer level. {\ourmodel} applies different quantization strategies to the key and value matrices, striking a balance between consumed memory and model performance.
    \item We conduct experiments on various datasets to substantiate the effectiveness of {\ourmodel}.
    Our results validate the asymmetric roles of the key and value matrices and demonstrate that by applying distinct quantization strategies to the key and value matrices, LLMs can be equipped with the extreme $1$ bit quantization while ensuring performance on par with the models utilizing floating-point parameters.
\end{itemize}

In the remainder of this paper, we first outline the basic definitions of transformers and KV cache in \secref{sec:pre}, then we highlight the observed asymmetric structural roles in \secref{sec:aas}, and present the design of {\ourmodel} in \secref{sec:layer}.
The evaluation and related works of AsymKV are discussed in \secref{sec:eval} and \secref{sec:rel}, respectively.
We finally conclude in \secref{sec:con}.
\section{Preliminaries}
\label{sec:pre}

\subsection{Attention Mechanism and KV Cache}
\label{sec:pre-attn}
Given the input embeddings of an attention mechanism, $\mathbf{X}\in \mathcal{R}^{t\times h}$, where $t$ represents the number of tokens already generated and $h$ is the dimension of attention head, an attention mechanism $\mathattn$~\cite{nips17attn,arxiv23gqa,arxiv19fasterattn} obtains the hidden states as follows: 
\begin{align}
\mathbf{Q}=\mathbf{X}\mathbf{W}^q,&\mathbf{K}=\mathbf{X}\mathbf{W}^k,\mathbf{V}=\mathbf{X}\mathbf{W}^v \notag\\
    &\mathbf{A}^w=sm(\frac{\mathbf{Q}\mathbf{K}^T}{\sqrt{h}}) \notag\\
    &\mathbf{A}^o=\mathbf{A}^w\mathbf{V} \notag 
\end{align}

\noindent Here, $\mathbf{W}^q$, $\mathbf{W}^k$ and $\mathbf{W}^v$ are the weight matrices for the query, key, and value, respectively, and
$sm(\cdot)$ signifies the softmax function.
$\mathbf{A}^w$ and $\mathbf{A}^o$ are typically referred to as the attention weights and attention output, respectively.

As an LLM generates tokens, the embeddings of the newly produced token are appended to the end of $\mathbf{X}$, necessitating the generation of query, key, and value matrices.
Consequently, we can store the embeddings of $\mathbf{K}$ and $\mathbf{V}$ from previous tokens and only generate the corresponding segments for the new token in $\mathbf{K}$ and $\mathbf{V}$.
Specifically, by partitioning $\mathbf{X}$ into the embeddings of previous tokens, \ie $\mathbf{X}_{1:t-1}$, and the embeddings of the current token, $\mathbf{X}_t$, we can leverage the key and value cache to enhance LLM's computational efficiency.
\begin{align}
    \mathbf{x}_q=\mathbf{X}_t \mathbf{W}^q,\text{ } \mathbf{x}_k=\mathbf{X}_t&\mathbf{W}^k,\text{ } \mathbf{x}_v=\mathbf{X}_t\mathbf{W}^v \notag\\
     \mathbf{K}_{1:t}=\mathcat(\mathbf{K}_{1:t-1},\mathbf{x}_k)&\text{, } \mathbf{V}_{1:t}=\mathcat(\mathbf{V}_{1:t-1},\mathbf{x}_v)\notag \\ \mathbf{A}^w&=\frac{\mathbf{x}_q\mathbf{K}_{1:t}^T}{\sqrt{h}} \label{eq:attn-qk}\\
     \mathbf{A}^w&=\mathsm(\mathbf{A}^w) \label{eq:attn-sm}\\
    \mathbf{A}^o&=\mathbf{A}^w\mathbf{V}_{1:t} \label{eq:attn-awv}
\end{align}
Here, the key and value matrices, $\mathbf{K}_{1:t-1}$ and $\mathbf{K}_{1:t-1}$ are cached while generating the last token.

\subsection{KV Cache Quantization}
\label{sec:pre-kvcq}

\stitle{Round-To-Nearest Quantization.}
While enhancing computational efficiency, the KV cache demands considerable memory, particularly as more tokens are generated.
To mitigate this, previous studies propose quantizing the key and value matrices into integers to accommodate more tokens using a Round-To-Nearest (RTN) methodology.

Formally, given a key or value matrix, $\mathbf{M}\in \mathcal{R}^{t\times h}$, an RTN quantization breaks down $\mathbf{M}$ into the quantized matrix $\mathbf{M}_Q$, the scaling matrix $\mathbf{s}$, and the zero-point matrix $\mathbf{z}$ as follows.

\begin{align}
    &\text{\textbf{Quantization Phase: }} \notag\\
    &\mathbf{z}=\min_i(\mathbf{M}), \text{ } \mathbf{s}=\frac{\max_i(\mathbf{M})-\min_i(\mathbf{M})}{2^b-1} \\
    & \mathbf{M}_Q=\lfloor \frac{\mathbf{M}-\mathbf{z}}{\mathbf{s}}\rceil \\
    & \text{\textbf{Dequantization Phase:}} \notag\\
    &\mathbf{M}^*=(\mathbf{M}_Q+\mathbf{z})*\mathbf{s} \label{eq:dequant}
\end{align}
Here, $b$ represents the required bit of quantization, and $\min_i$ (respectively $\max_i$) is a function that retrieves the minimum (respectively maximum) tensor of the input in relation to the $i$-th dimension. 
$i$ may be chosen from $\{1,2\}$, representing \emph{per-channel} or \emph{per-token} quantization respectively.

\stitle{Measurement of Quantization.}
Given a quantization method, a natural question would be how to measure the effectiveness of the proposed method.
Recent works~\cite{arxiv22gptq,arxiv24qaq} proposed using the squared error of the output between the quantized weights and full-precision weights to measure the effectiveness or optimize the strategies.
Formally, the error is 
\begin{align}
    &e=|||f(\mathbf{M}^*)-f(\mathbf{M})||_2^2
\end{align}

\noindent where $f(\cdot)$ could be a linear layer or the whole attention layer (\ie \equref{eq:attn-qk}-\equref{eq:attn-awv}).
Following these works, we use the squared error to study how the structure of attention affects the effectiveness.

\section{Asymmetric Attention Sensitivity of KV Cache Quantization}
\label{sec:aas}

\begin{figure}
    \centering
    \includegraphics[width=1\linewidth]{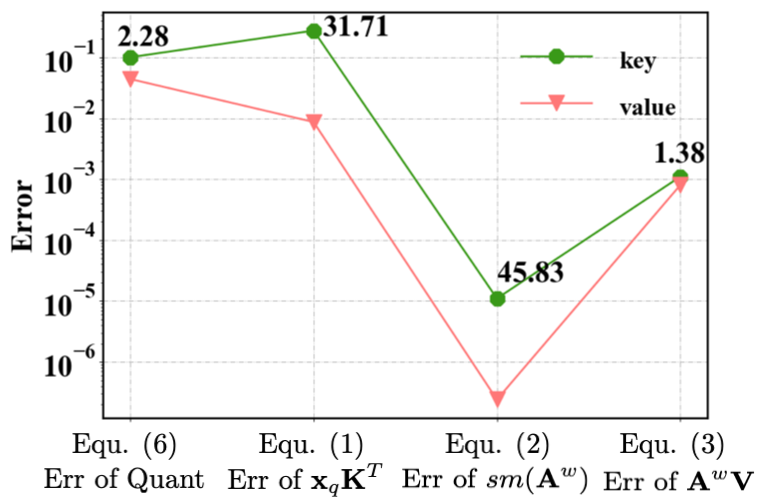}
    \caption{Squared error in the inference of attention.}
    \label{fig:loss-var}
\end{figure}
As shown in \equref{eq:attn-qk}-\equref{eq:attn-awv}, the key matrix and value matrix perform distinct roles in transformers.
While existing studies have proposed intricate quantization methods to mitigate the loss from quantization and some studies~\cite{arxiv24qaq} have recognized the disparate roles of the key matrix and value matrix, an important question still lingers: provided that the key matrix and value matrix play different roles from various perspectives, for instance, the multiplication of $\mathbf{x}_q$ and the operation of softmax function on key matrix, \emph{what factor truly contributes to the loss of the transformer?}

\stitle{Observation.}
For the key (respectively value) matrix, we hold the value (respectively key) matrix in floating type, and evaluate the \emph{accumulated} mean squared error between the output with key (respectively value) matrix in floating type and that with $2$-bit quantization at different stages of the attention.
\figref{fig:loss-var} illustrates the average loss per element during the inference of the Llama-2 model of size 7b.
Here, the green (respectively red) line denotes the MSE between the attention output with floating type and the $2$-bit quantization of the key (respectively value) matrix in different stages of the attention.
The number on the lines depicted in \figref{fig:loss-var} represents the ratio between the MSE that arises from the key matrix quantization and the MSE that arises from the value matrix quantization.

Interestingly, even though the quantization strategy results in a comparative loss (\ie the MSE after \equref{eq:dequant}) on the key matrix and value matrix, there emerges marginal gap loss for key matrices after the multiplication of $\mathbf{x}_q$, \ie after \equref{eq:attn-qk}.
The gap is further amplified after the softmax function, \ie after \equref{eq:attn-sm}.
This indicates that even though the quantization methods can guarantee a similar loss for key and value matrices, the multiplication of $\mathbf{x}_q$ and the activation function makes the MSE of the attention output for the key matrix significantly larger than that of the value matrix.

\begin{figure*}[t!]
  \centering
  \subfloat[Layer 21 Head 13]{\includegraphics[width=0.33\textwidth]{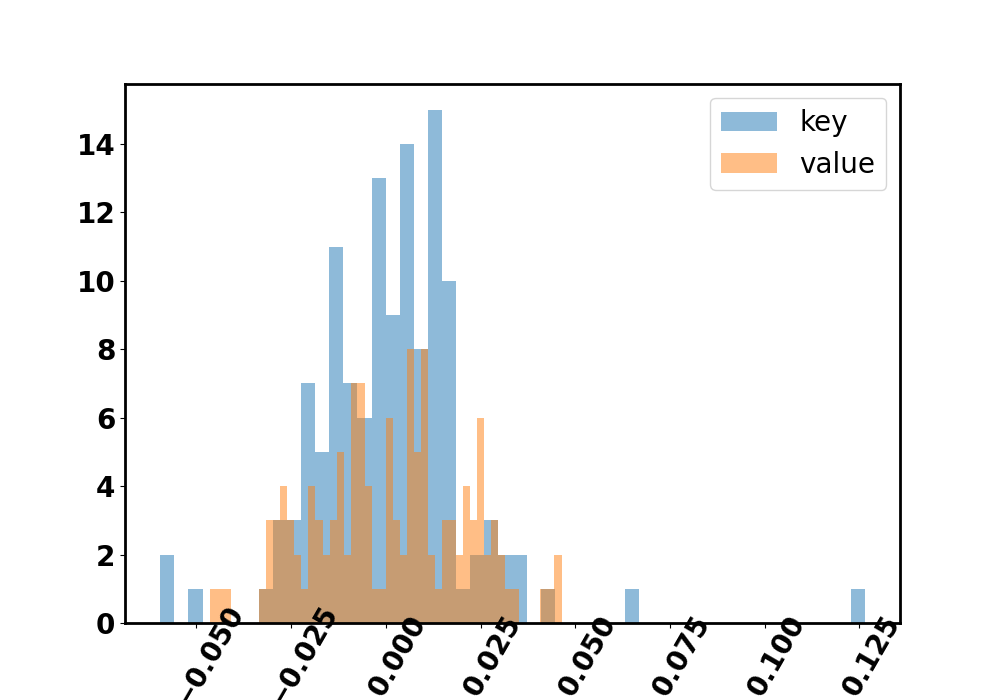}}
  \hfill
  \subfloat[Layer 30 Head 25]{\includegraphics[width=0.33\textwidth]{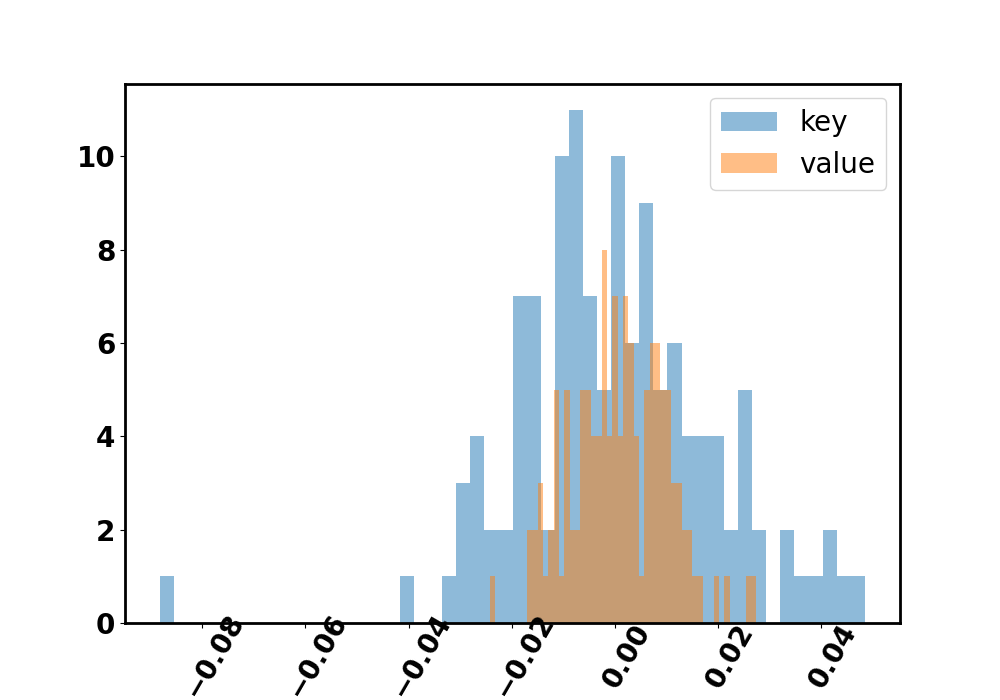}}
  \hfill
  \subfloat[Layer 31 Head 22]{\includegraphics[width=0.33\textwidth]{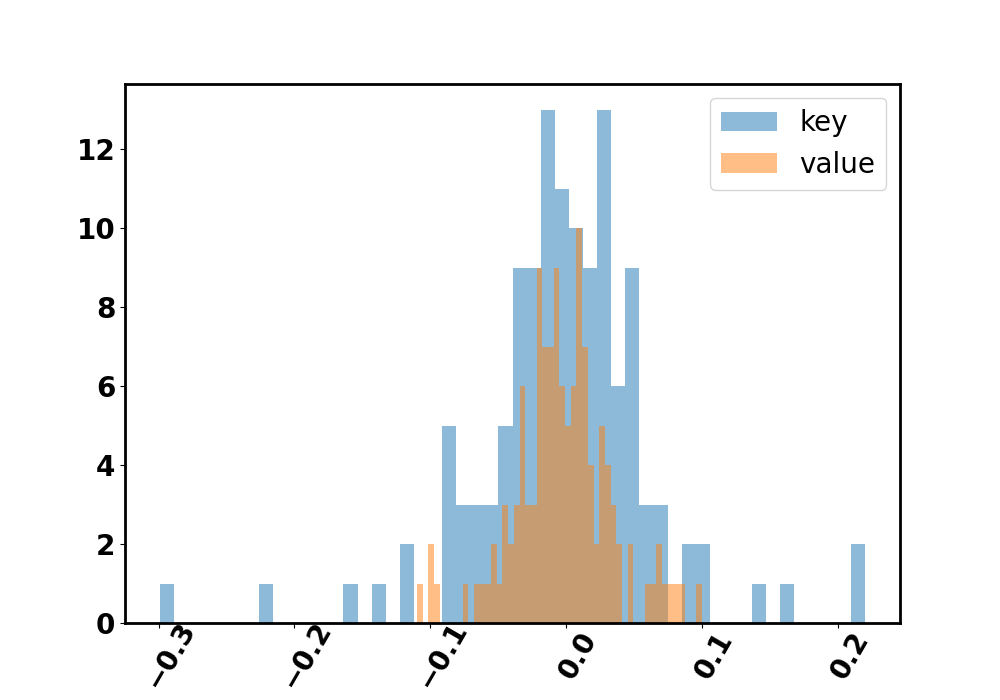}}
  \caption{Statistics of the error from key matrix quantization and value matrix quantization.}
  \vspace{-4pt}
  \label{fig:scatter-mse}
\end{figure*}

\stitle{MSE Amplification.}
Next, we analyze why the multiplication of $\mathbf{x}_q$ and the softmax function exacerbates the MSE of the key matrix.
Consider a matrix $\mathbf{M}$ and its quantization matrices, $\mathbf{M}_Q$, $\mathbf{z}$, and $\mathbf{s}$.
$\mathbf{M}$ could be either the key matrix or the value matrix.
Assume that the deviation of each element between $\mathbf{M}$ and the quantized matrix follows the distribution $\mathbf{P}$, \ie $|\mathbf{M}_{i,j}-\mathbf{M}^*_{i,j}|\sim\mathbf{P}$.
We aim to understand how the error of an element in the matrix varies after being multiplied by a vector.
\begin{proposition}
    \label{prop:err-mul}
    Consider a matrix $\mathbf{M}$ and its estimation $\mathbf{M}^*$.
    Denote the error by $\mathbf{E}=\mathbf{M}-\mathbf{M}^*$.
    Upon left multiplying by a matrix $\mathbf{A}$, the error matrix becomes $\mathbf{A}\mathbf{E}$.
    Correspondingly, a right multiplication of $\mathbf{A}$ results in the error $\mathbf{E}\mathbf{A}$.
\end{proposition}

\begin{proof}
Consider the $(s,r)$-th element of $\mathbf{A}\mathbf{M}$.
We could obtain its error
\begin{align}
    \label{eq:val-error}
    &\mathbf{A}_{s,\cdot}\mathbf{M}_{\cdot,r}-\mathbf{A}_{s,\cdot}\mathbf{M}^*_{\cdot,r} \notag \\
    =&\sum_i\mathbf{A}_{s,i}(\mathbf{M}_{i,r}-\mathbf{M}_{i,r}^*) \notag \\
    =&\sum_{i}\mathbf{A}_{s,i}\epsilon_{i,r}
\end{align}
which precisely corresponds to the $(s,r)$-th element of $\mathbf{A}\mathbf{E}$.
Similarly, the right multiplication of $\mathbf{A}$ results in an error matrix $\mathbf{E}\mathbf{A}$.
\end{proof}

Based on \propref{prop:err-mul}, we can deduce the error stemming from the value matrix's quantization.
\begin{proposition}
    \label{prop:err-val}
    Given a value matrix $\mathbf{V}$ and its quantization $\mathbf{V}^*$, with a quantization error $\mathbf{E}^v=\mathbf{V}-\mathbf{V}^*$, the error in the attention output is $\mathbf{A}^w\mathbf{E}^v$.
\end{proposition}

\propref{prop:err-val} can be derived from \equref{eq:attn-awv} and \propref{prop:err-mul}.
On the other hand, it is also feasible to derive the error resulting from the quantization of key matrices, although this process is complex due to the involvement of softmax functions.

\begin{theorem}
\label{thm:key-err}
    Given a key matrix $\mathbf{K}$ and its quantization $\mathbf{K}^*$, with a quantization error $\mathbf{E}^k=\mathbf{K}-\mathbf{K}^*$, the error of the attention output is given by $(\mathbf{A}^w\odot (1-sr\cdot e^{\frac{\mathbf{E}^q}{\sqrt{h}}})\cdot \mathbf{V}$, where $\mathbf{E}^q=\mathbf{x}_q\mathbf{E}^k$, $\min\mathbf{E}^q$ and $\max\mathbf{E}^q$ are the smallest and largest elements of $\mathbf{E}^q$ respectively, and $sft=\sum_{j}e^{\frac{\sum_iq_i\mathbf{K}_{i,j}}{\sqrt{h}}}$ and $sft^*=\sum_{j}e^{\frac{\sum_iq_i\mathbf{K}_{i,j}^*}{\sqrt{h}}}$ are the dominator in the softmax function for the key matrix $\mathbf{K}$ and $\mathbf{K}^*$ respectively.
\end{theorem}
\begin{proof}
    Consider the error in the $1,r$-th element of $\mathbf{A}^w$.
    It is given by
    \begin{align}
        \label{eq:err-key-aw}
        &\frac{e^{\frac{\sum_i q_i\mathbf{K}_{i,r}}{\sqrt{h}}}}{sft}-\frac{e^{\frac{\sum_i q_i\mathbf{K}_{i,r}^*}{\sqrt{h}}}}{sft^*} \notag \\
        =& \frac{e^{\frac{\sum_i q_i\mathbf{K}_{i,r}}{\sqrt{h}}}}{sft}(1-\frac{sft}{sft^*}\frac{e^{\frac{\sum_i q_i\mathbf{K}_{i,r}^*}{\sqrt{h}}}}{e^{\frac{\sum_i q_i\mathbf{K}_{i,r}}{\sqrt{h}}}}) \notag \\
        =&\mathbf{A}^w_{1,r}(1-\frac{sft}{sft^*}e^{\frac{\sum_iq_i(\mathbf{K}_{i,r}^*-\mathbf{K}_{i,r})}{\sqrt{h}}})\notag \\
        =&\mathbf{A}^w_{1,r}(1-\frac{sft}{sft^*}e^{\frac{-\mathbf{x}_q\mathbf{E}^k_{r}}{\sqrt{h}}})\notag \\
        =&\mathbf{A}^w_{1,r}(1-\frac{sft}{sft^*}e^{\frac{\mathbf{E}^q}{\sqrt{h}}})
    \end{align}
    This can be reformulated in matrix form as $\mathbf{A}_{1,r}^w\odot(1-sr\cdot e^{\frac{\mathbf{E}^q}{\sqrt{h}}})$.
    Since $\mathbf{A}^w$ is subsequently multiplied by $\mathbf{V}$, in accordance with \propref{prop:err-mul}, the error in the attention output is given by
    \begin{align}
    \label{eq:key-error-final}
    (\mathbf{A}_{1,r}^w\odot(1-sr\cdot e^{\frac{\mathbf{E}^q}{\sqrt{h}}}))\cdot\mathbf{V}.
    \end{align}
\end{proof}

To demonstrate the difference in error caused by the quantization of the key and value matrix, we select three decoder layers and plotted the error from \equref{eq:val-error} and \equref{eq:key-error-final} in \figref{fig:scatter-mse}.
The results indicate that the distribution of the key matrix quantization error is more sparse around $0$ compared to the value matrix quantization, which consequently leads to a larger MSE for the key matrix.

\begin{figure*}[t!]
    \centering
    \includegraphics[width=0.95\linewidth]{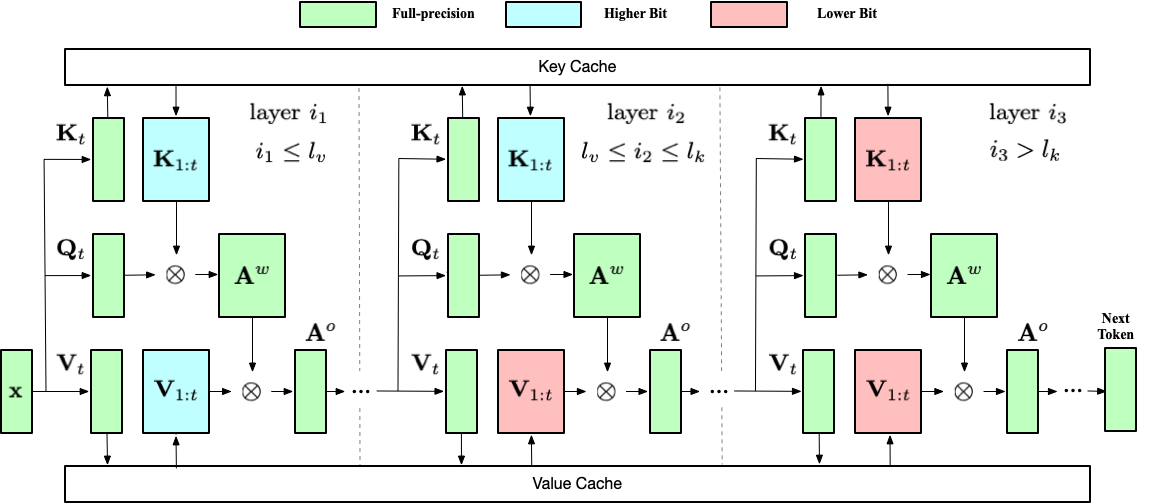}
    \caption{Workflow of \ourmodel.}
    \label{fig:kvmix}
\end{figure*}

\stitle{Discussion: Why does the key matrix quantization leads to a larger error than the value matrix?}
For the value matrix, it is not influenced by the softmax function, making its error straightforward to compute and directly tied to the quantization error.
In contrast, for the key matrix quantization as shown in \equref{eq:key-error-final}, $sft=\sum_{j}e^{\frac{\sum_iq_i\mathbf{K}_{i,j}}{\sqrt{h}}}$ and $sft^*=\sum_{j}e^{\frac{\sum_iq_i\mathbf{K}_{i,j}^*}{\sqrt{h}}}$ are relatively large, and they are nearly equivalent because $\mathbf{K}^*$ is the quantization of $\mathbf{K}$.
This suggests that $sr\approx 1$ and the key discrepancy between the errors of the key matrix quantization and value matrix quantization arises in the Hadamard product of $1-sr\cdot e^{\frac{\mathbf{E}^q}{\sqrt{h}}}$.
(1) Multiplication of $\mathbf{x}_q$.
Observe that the first dimension of $\mathbf{x}_q$ is consistently set to 1.
Thus, the multiplication by $\mathbf{x}_q$ results in each element accumulating the error from quantization multiple times.
This is illustrated in \equref{eq:key-error-final}, where each element of $\mathbf{E}^q$ has a comparatively larger error than the error distribution of $\mathbf{P}$, given that $\mathbf{E}^q=q\mathbf{E}^k$.
(2) Utilization of the softmax function.
In the key matrix error obtained in \equref{eq:key-error-final}, the original error from the key matrix quantization is situated in the exponentiation of $e$.
As the proof of \thmref{thm:key-err} demonstrates, this replacement stems from the utilization of the softmax function, in which all elements are treated as the exponentiation of $e$.
In consideration of the super-linear growth rate of the power function, the softmax function further exacerbates the loss induced by the key matrix quantization.

\section{\ourmodel: Layer-wise Quantization with Asymmetric Configuration}
\label{sec:layer}
From the discussion in \secref{sec:aas}, it is evident that the quantization of the key matrix could potentially result in a more significant loss for the attention output due to the specific role of the key matrix.
In response to this, our study introduces {\ourmodel}, a simple yet efficacious quantization strategy that blends various degrees of quantization for the key and value matrix based on their respective impacts on the loss of the attention mechanism.

\stitle{Basic Idea.}
{\ourmodel} applies various degrees of quantization to the key and value matrix at the layer level.
Specifically,
{\ourmodel} introduces two parameters, $l_k$ and $l_v$, to control the degree of quantization for the key and value matrix, respectively.
During the inference of the model, for the key (respectively value) matrix, the initial $l_k$ (respectively $l_v$) attention layers utilize a quantization method with a higher number of bits (\eg 4-bit or 2-bit). In contrast, the remaining attention layers employ a quantization method with fewer bits (\ie 1-bit).

\figref{fig:kvmix} illustrates the design of {\ourmodel} where green, blue, and red blocks symbolize the matrices in full-precision, higher-bit quantization, and lower-bit quantization, respectively.
For each attention layer, its key and value matrices are cached with different quantization bits based on the layer index.
As demonstrated in \figref{fig:kvmix}, those layers with a layer index $i\leq l_k$ (respectively $i\leq l_v$) will cache the quantized key (respectively value) matrices with higher bits, while the other layers will use lower bits.
After generating the query, key, and value matrix of the current token, \ie $\mathbf{K_t}$, $\mathbf{Q_t}$, and $\mathbf{V_t}$, the LLM will produce the output of the attention $\mathbf{A}^o$, as illustrated in \equref{eq:attn-qk}-\equref{eq:attn-awv}.
Given that {\ourmodel} chooses $l_k$ and $l_v$ such that $l_v\leq l_k$, those decoder layers with indices in range $[l_v,l_k]$ will contain a blend of higher bits for key matrix and lower bits for value matrix.

The design of {\ourmodel} relies on the observations in \secref{sec:aas} as well as certain intuitive insights.

\stitle{(1) Asymmetric Configuration.}
In light of our observation in \secref{sec:aas}, we decide to independently configure the degree of quantization for key and value matrices by defining the configuration parameters $l_k$ and $l_v$, respectively.
Besides, since the quantization error for the key matrix results in a larger error for the attention output, we generally choose a larger $l_k$ than $l_v$ to achieve performance comparable to the models with full precision.

\stitle{(2) Layer-wise Quantization.}
While generating a token, the quantization error is accumulated as the number of attention layers increases.
Therefore, by choosing the later attention layers to be quantized with fewer bits, we can mitigate the error caused by the quantization from being amplified, while concurrently allowing the KV cache to be quantized with a less number of bits.


\section{Evaluation}
\label{sec:eval}

\subsection{Experimental Setup}
\stitle{Tested Models.}
We examine the performance of {\ourmodel} using the widely used LLM family Llama~\cite{arxiv23llama2}, which includes Llama-2-7b and Llama-2-13b.
All models are deployed based on the LLM implementation from Huggingface\footnote{https://huggingface.co/} with the default implementation of quantization selected from \cite{icml24kivi}.

\stitle{Tasks and Baselines.}
In terms of model performance, we evaluate {\ourmodel} on tasks with a standard context length, including CoQA and TruthfulQA from LM-Eval~\cite{lmeval}, as well as tasks with long context length from LongBench~\cite{arxiv23longbench}, including TriviaQA, TREC, SAMSum, RepoBench-P, and Qasper.
Regarding model efficiency, we assess the memory usage of {\ourmodel} under various quantization configurations, comparing it with previous works that handle the key and value matrices uniformly, including the original floating implementation, and KIVI~\cite{icml24kivi} with 2-bit quantization.

\stitle{Implementation.}
Each decoder layer in {\ourmodel} adheres to the quantization scheme outlined in KIVI~\cite{icml24kivi}, that is, per-channel quantization for the key matrix and per-token quantization for the value matrix, with a group size of 32.
{\ourmodel} utilizes a combination of higher 2-bit quantization and lower 1-bit quantization.
To validate our analysis concerning the diverse errors instigated by the key matrix quantization and value matrix quantization, we also examine {\ourmodel} under various quantization configurations.

\begin{table}
    \centering
    \small{
    \begin{tabular}{cccc}
    \toprule
        Model & Type & TruthfulQA & CoQA \\ \midrule
        \multirow{3}{*}{Llama-2-7b} & float & 30.76 & 63.88 \\ 
         ~ & KIVI-2bit & 33.95 & 63.05 \\ \cmidrule{2-4}
         ~ & AsymKV-0/16 & 12.81 & 34.18 \\ 
         ~ & AsymKV-16/0 & \textbf{38.77}* & \textbf{58.12}* \\ \midrule
        \multirow{3}{*}{Llama-2-13b} & float & 29.53 & 66.37 \\ 
        ~ & KIVI-2bit & 29.84 & 66.23 \\ \cmidrule{2-4}
        ~ & AsymKV-0/20 & 9.52 & 43.13 \\ 
        ~ & AsymKV-20/0 & \textbf{28.44}* & \textbf{61.42}* \\ \bottomrule
    \end{tabular}
    }
    \caption{Evaluation on tasks with normal context length (\textbf{bold}: Higher bits for key matrix better than lower bits for key matrix, *: {\ourmodel} achieves at least 90\% performance of floating-type models).}
    \label{tab:exp-normal}
\end{table}

\begin{table*}
    \centering
    \small{
    \begin{tabular}{cccccccc}
    \toprule
        Model & Type & TriviaQA & TREC & SAMSum & RepoBench-P & Qasper \\ \midrule
        \multirow{3}{*}{Llama-2-7b} & float & 87.72 & 66.0 & 41.69 & 59.82 & 9.52 \\ 
         & KIVI-2bit & 87.64 & 66.0 & 41.62 & 56.81 & 9.73 \\ \cmidrule{2-7}
         &AsymKV-0/32 & 11.6 & 25.0 & 3.79 & 23.9 & 3.18 \\ 
         & AsymKV-32/0 & \textbf{85.27}* & \textbf{65.50}* & \textbf{38.28}* & \textbf{43.35} & \textbf{8.96}* \\ \midrule
        \multirow{3}{*}{Llama-2-13b} & float & 87.87 & 70.00 & 43.55 & 56.42 & 9.32 \\
         & KIVI-2bit & 87.31 & 69.50 & 43.52 & 53.66 & 8.27 \\ \cmidrule{2-7}
         & AsymKV-0/40 & 24.57 & 28.5 & 5.25 & 25.33 & 3.33 \\ 
         & AsymKV-40/0 & \textbf{86.70}* & \textbf{67.50}* & \textbf{41.90}* & \textbf{46.92} & \textbf{8.78}* \\ \bottomrule
    \end{tabular}
    }
    \caption{Evaluation on LongBench tasks (\textbf{bold}: Higher bits for key matrix better than lower bits for key matrix, *: {\ourmodel} achieves at least 90\% performance of floating-type models).}
    \label{tab:exp-long}
\end{table*}

\subsection{Evaluation Results}
\subsubsection{Tasks with Normal Context Length}
\label{subsec:eval-normal}
\tabref{tab:exp-normal} presents the experimental results for tasks with normal context length, namely CoQA and TruthfulQA.
In this case, the model {\ourmodel}-$l_k$/$l_v$ represents {\ourmodel} where the key and value matrices in the first $l_k$ and $l_v$ attention layers are respectively quantized with $2$-bit, while those in other layers are quantized with $1$ bit.

Upon examining {\ourmodel} with various quantization configurations, we observe that {\ourmodel}-16/0 (respectively {\ourmodel}-20/0) performs better than {\ourmodel}-0/16 (respectively {\ourmodel}-0/20) for Llama-7b (respectively Llama-13b).
This finding aligns with our observation and analysis in \secref{sec:aas}, where the quantization of key matrices results in a higher loss than that of value matrices.
Therefore, even though {\ourmodel}-16/0 and {\ourmodel}-0/16 occupy the same space in GPU memory, a quantization strategy that employs higher bits for the key matrix and lower bits for the value matrix enhances performance.

Besides, {\ourmodel} yields performance comparable to Llama and KIVI while using less GPU memory, achieved by implementing asymmetric $1$-bit quantization.
In particular, {\ourmodel}-16/0 and {\ourmodel}-20/0 assures a minimum performance of $91.0$\% that of Llama and $92.2$\% that of KIVI.
In contrast to KIVI, which quantizes both key and value matrices with $2$ bits, {\ourmodel} allows for 75\% decoder layers quantized with the extreme 1 bit, which is more efficient in peak memory.

\subsubsection{Tasks with Long Context Length}
\tabref{tab:exp-long} presents the experimental results for tasks with long context lengths.
Mirroring the tasks with normal context length, {\ourmodel} with a higher bit count in the key matrix (\ie {\ourmodel}-32/0 for Llama-7b and {\ourmodel}-40/0 for Llama-13b) once more surpasses {\ourmodel} with value matrices quantized with higher bits (\ie {\ourmodel}-0/32 and {\ourmodel}-0/40).
This aligns with the reasons as illustrated in \secref{subsec:eval-normal}.

Besides, in the case of long context lengths, {\ourmodel} necessitates more decoder layers quantized with higher bits to attain performance comparable to the baselines ($l_k=32/40$ for long context length vs. $l_k=16/20$ for normal context length).
When contrasted with the baselines, {\ourmodel} could assure performance levels of at least 91.8\% and 92.0\% relative to Llama and KIVI across $4$ out of $5$ datasets, except for RepoBench-P.

\begin{figure*}[t!]
  \centering
  \subfloat[{\ourmodel} for Llama-7b (batch size 48).]{\includegraphics[width=0.50\textwidth]{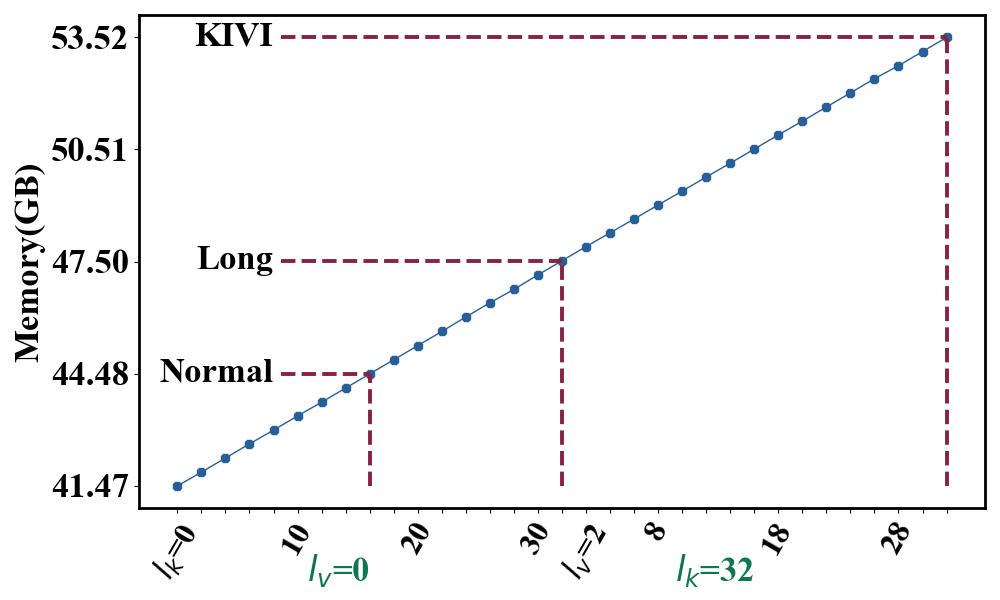}\label{fig:mem-7b}}
  \hfill
  \subfloat[{\ourmodel} for Llama-13b (batch size 36).]{\includegraphics[width=0.50\textwidth]{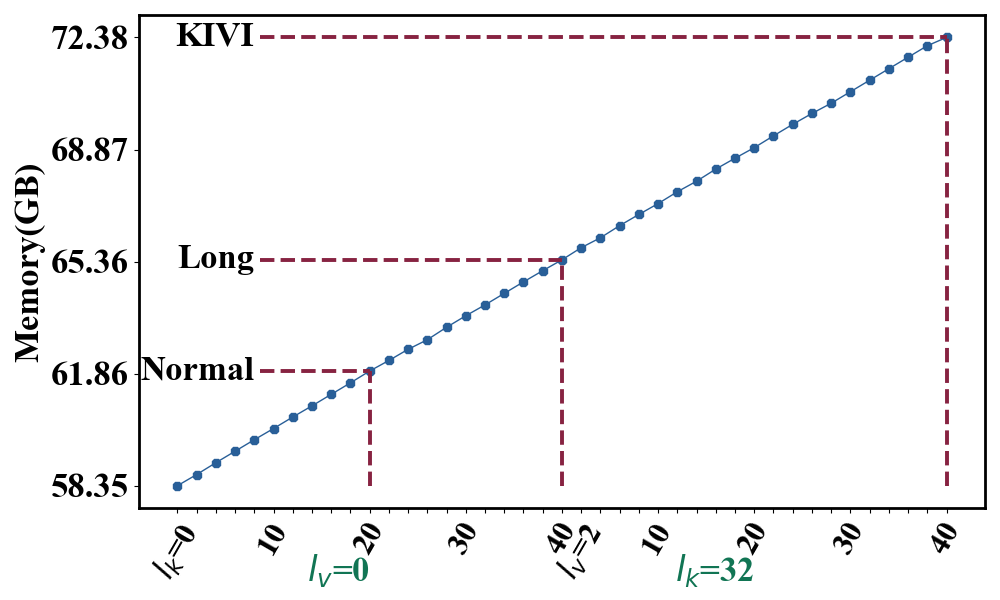}\label{fig:mem-13b}}
    
  \caption{Memory Variation of {\ourmodel}.}
  \label{fig:exp-mem}
\end{figure*}

\subsubsection{Peak Memory}
\figref{fig:exp-mem} reports the experimental results of the peak memory in GPU for {\ourmodel}.
We choose a batch size of $48$ for Llama-7b and $36$ for Llama-13b, and report the peak storage consumption by varying the quantization configurations $l_k$ and $l_v$.
Specifically, we first set $l_v=0$, implying all value matrices of the decoder layers are quantized with $1$ bit, and increase the number of key matrices quantized with $2$ bits, \ie $l_k$, from $0$ to the maximum number of decoder layers, illustrated in the left part of \figref{fig:mem-7b} and \figref{fig:mem-13b}.
Then, we keep all key matrices quantized with $2$ bits and further increase the number of value matrices quantized with $2$ bits, \ie $l_v$, as shown in the right part of \figref{fig:mem-7b} and \figref{fig:mem-13b}.
It is noteworthy that when both $l_k$ and $l_v$ achieve the maximum number of layers, the results correspond to the performance of KIVI.

From \figref{fig:exp-mem}, as more attention layers are quantized with higher bits, the consumed space in GPU increases almost linearly until all attention layers employ a quantization configuration with higher bits.
The locations where {\ourmodel} achieves comparable performance to the floating-point model on tasks with normal and long context lengths are highlighted.
 For Llama-7b, {\ourmodel} can ensure similar performance while saving 9.0 GB and 6.0 GB of space for the tasks with normal and long context lengths respectively, compared to KIVI.
 For Llama-13b, the memory saved increases to 10.4GB and 7.0GB space for tasks with normal and long context lengths respectively.
 
\section{Related Works}
\label{sec:rel}

Large language models have gained considerable attention since their inception.
Despite their impressive performance, these models are constrained by their immersive quantity of parameters, which results in hardware limitations and poor throughput.

To address these issues, recent research trends are centered on reducing the size of LLMs~\cite{arxiv23squeeze}.
Among these methods, quantization techniques target the transformation of a portion of the model's parameters into integers, which reduces the space of LLMs.
For instance, llm.int8~\cite{nips22int8} suggests quantizing the query, key, and value weights of LLMs using the round-to-nearest method, \ie, mapping each floating-point number to its closest integer.
AWQ~\cite{mlsys24awq} and SmoothQuant~\cite{icml23smoothquant} further introduce an amplifying scale prior to quantization to prevent extremely large outliers during the process.
Omniquant~\cite{arxiv23omniquant} devises a quantization algorithm by implementing a learnable scale and learnable clipping during quantization.
GPTQ~\cite{arxiv22gptq} perceives quantization as a problem of minimizing square error and designing the quantization algorithms using an approximation of the second-order information.
These studies mainly focus on the quantization of the model weights.

On the other hand, to mitigate redundant computations across token generation, LLMs utilizes KV cache.
While KV cache enhances inference efficiency, it consumes significant space, particularly when generating long contexts.
Consequently, another line of research focuses on the compression of KV cache~\cite{nips24h2o,kwon2023efficient,nips23s3,nips24scissor}.
Among these approaches, quantization techniques have garnered much attention and have emerged as a popular tool for KV cache compression.

Previous works have applied consistent quantization techniques for both model weights and KV cache.
SmoothQuant~\cite{icml23smoothquant} also quantizes the query, key, and value matrices to further minimize memory usage.
In contrast, Flexgen~\cite{icml23flexgen} structures the problem of quantization in an environment conprising GPU, CPU, and memory.
ATOM~\cite{zhao2024atom} utilizes the quantized weights and re-quantizes the key and value cache matrices into integer types.

More recently, several works have examined the distriution of the KV cache and formulated quantization algorithms specifically tailored for it.
For example, ATOM~\cite{zhao2024atom} discovered that the key matrix contains more outliers than the value matrix.
KIVI~\cite{icml24kivi} extends on this observation and suggesting quantizing the key and value matrices from different perspectives (employing per-channel quantization for key matrix and per-token quantization for value matrix).
Alongside this, KVQuant employs the per-channel quantization to the key matrices and introduces a non-uniform quantization technique for KV cache.
IntactKV~\cite{arxiv24intact} identifies outliers caused by common tokens, including the punctuations and split tokens.
Meanwhile, WKVQuant~\cite{arxiv24wkvquant} proposes a two-dimensional quantization strategy to smoothly handle the outliers across different channels.
Other studies seek to combine the quantization techniques with other techniques to approach a fine-grained KV cache compression.
GEAR~\cite{arxiv24gear} establishes the residual matrix and sparse matrix to capture the residual and individual outliers during the quantization.
\cite{arxiv24notoken} proposes a mix-precision quantization scheme and quantizes the crucial KV cache with higher bits.
In contrast to the aforementioned studies, 
We propose a simple yet effective solution: employing distinct quantization configurations for the key and value matrices at the layer level.
This approach is designed to accommodate the asymmetric role of the key and value matrices.

\section{Conclusions}
\label{sec:con}
This paper primarily concentrates on the asymmetric roles of the key and value matrices in the quantization of the KV cache.
We analyze why quantizing the key matrix leads to a more significant performance drop than quantizing the value matrix and attribute it to the multiplication of $\mathbf{x}_q$ and the implementation of the softmax function.
Based on this analysis, we introduce {\ourmodel}, which applies asymmetric and layer-wise quantization configurations to the key and value matrices.
{\ourmodel} facilitates a mixed quantization approach with $2$ bits and $1$ bit, while simultaneously ensuring performance comparable to the floating-type model.
Extensive experiments validate our analysis of the asymmetric roles of the key and value matrices.

\section{Limitations}
Despite {\ourmodel} facilitating the quantization of $1$ bit for KV cache in LLMs, it still depends on exhaustive testing to identify the optimal configurations for different LLMs, \ie configurations that yield performance close to models in floating types.
This approach is relatively inefficient.
A potential futural direction could involve efficiently identifying the optimal configurations for LLMs.

\bibliography{asymKV}

\clearpage
\newpage
\appendix

\section{Supplemental Experiments}
\label{sec:appendix}

In this section, we present the complete experimental setup and the corresponding results.

\subsection{Experimental Settings}
\stitle{Inference Settings.} 
Following KIVI~\cite{icml24kivi}, {\ourmodel} employs per-channel quantization for key matrices and per-token quantization for value matrices.
Consequently, both KIVI and {\ourmodel} store the key matrices of a limited number of tokens in floating-point types, a parameter referred to as residual length.
We choose a residual length of $128$ for tasks with normal context length, while for tasks with long context length, we opt for a residual length of $512$.
For the peak memory usage experiments, we standardized the generation length of tokens to $4096$.

\stitle{Implementation.}
{\ourmodel} is implemented using PyTorch and is built upon the Huggingface codebase. 
All 
experiments are executed on a machine equipped with 200GB memory and an A800 GPU with 80GB memory.

\subsection{Experimental Results}
\subsubsection{Results on Tasks with Normal Context Length}
\tabref{tab:full-exp-normal} proposes the performance of {\ourmodel} with varying $l_k$ and $l_v$ values for tasks with normal context length.
For Llama-7b, we choose $l_k,l_v\in\{6,12,16,20\}$ and for Llama-13b, we consider $l_k,l_v\in\{5,10,20,30\}$.

As the number of decoder layers quantized with higher bits increases, the performance of {\ourmodel} improves until it reaches performance levels comparable to the floating-point model and KIVI.
Besides, we observe that {\ourmodel} with value matrices quantized using lower bits, \ie {\ourmodel}-$l$/0, consistently outperforms {\ourmodel} with key matrices quantized using lower bits, \ie {\ourmodel}-0/$l$, and the difference is substantial.
This observation confirms that choosing a configuration with $l_k>l_v$ can enhance the performance of {\ourmodel}.
{\ourmodel} can achieve at least 90\% of the performance of floating-point models when a quantization configuration that follows {\ourmodel}-16/0 for Llama-7b and {\ourmodel}-20/0 for Llama-13b is utilized.

\begin{table}
    \centering
    \small{
    \begin{tabular}{cccc}
    \toprule
        Model & Type & TruthfulQA & CoQA \\ \midrule
        \multirow{10}{*}{Llama-2-7b} & float & 30.76 & 63.88 \\ 
         ~ & KIVI-2bit & 33.95 & 63.05 \\ \cmidrule{2-4}
         ~ & AsymKV-0/6 & 4.11 & 26.90 \\
         ~ & AsymKV-0/12 & 7.37 & 28.92 \\
         ~ & AsymKV-0/16 & 12.81 & 34.18 \\ 
         ~ & AsymKV-0/22 & 12.23 & 35.60 \\ \cmidrule{2-4}
         ~ & AsymKV-6/0 & 7.64 & 36.00 \\
         ~ & AsymKV-12/0 & 29.17* & 48.02 \\
         ~ & AsymKV-16/0 & 38.77* & 58.12* \\ 
         ~ & AsymKV-22/0 & 40.14* & 59.83* \\ \midrule
        \multirow{10}{*}{Llama-2-13b} & float & 29.53 & 66.37 \\ 
        ~ & KIVI-2bit & 29.84 & 66.23 \\ \cmidrule{2-4}
        ~ & AsymKV-0/5 & 4.81 & 37.53 \\
        ~ & AsymKV-0/10 & 4.16 & 39.70 \\
        ~ & AsymKV-0/20 & 9.52 & 43.03 \\ 
        ~ & AsymKV-0/30 & 10.24 & 45.20 \\ \cmidrule{2-4}
        ~ & AsymKV-5/0 & 15.35 & 41.25 \\
        ~ & AsymKV-10/0 & 19.43 & 45.40 \\
        ~ & AsymKV-20/0 & 28.44* & 61.42* \\ 
        ~ & AsymKV-30/0 & 29.50* & 64.92* \\ \bottomrule
    \end{tabular}
    }
    \caption{Evaluation on tasks with normal context length (*: {\ourmodel} achieves at least 90\% performance of floating-type models).}
    \label{tab:full-exp-normal}
\end{table}

\subsubsection{Results on Tasks with Long Context Length}
\tabref{tab:full-exp-long} presents the experimental results for tasks with long context length.
For key and value matrices, we set aside one type of matrices quantized with higher bits (\ie $l_k/l_v=32/40$) and vary the number of the other type of matrices that are quantized with lower bits.

Similar to the tasks with normal context lengths, the performance of {\ourmodel} augments as more key and value matrices are quantized with higher bits.
Besides, {\ourmodel} with key matrices quantized with higher bits ({\ourmodel}-32/$l_v$ for Llama-7b and {\ourmodel}-40/$l_v$ for Llama-13b) outperforms {\ourmodel} with value matrices quantized with higher bits, despite them occupying the same GPU memory.

\begin{table*}
    \centering
    \small{
    \begin{tabular}{cccccccc}
    \toprule
        Model & Type & TriviaQA & TREC & SAMSum & RepoBench-P & Qasper \\ \midrule
        \multirow{8}{*}{Llama-2-7b} & float & 87.72 & 66.0 & 41.69 & 59.82 & 9.52 \\ 
         & KIVI-2bit & 87.64 & 66.0 & 41.62 & 56.81 & 9.73 \\ \cmidrule{2-7}
         &AsymKV-0/32 & 11.6 & 25.0 & 3.79 & 23.9 & 3.18 \\ 
         & AsymKV-6/32 & 19.02 & 29.0 & 5.53 & 28.46 & 4.04 \\
         &AsymKV-12/32 & 22.96 & 42.50 & 8.77 & 32.34 & 5.13 \\ \cmidrule{2-7}
         & AsymKV-32/0 & 85.27* & 65.50* & 38.28* & 43.35 & 8.96* \\
         & AsymKV-32/6 & 86.36* & 66.50* & 39.75* & 49.93 & 9.04* \\
         & AsymKV-32/12 & 86.62* & 66.00* & 40.93* & 52.46 & 9.64* \\
         \midrule
        \multirow{8}{*}{Llama-2-13b} & float & 87.87 & 70.00 & 43.55 & 56.42 & 9.32 \\
         & KIVI-2bit & 87.31 & 69.50 & 43.52 & 53.66 & 8.27 \\ \cmidrule{2-7}
         & AsymKV-0/40 & 24.57 & 28.5 & 5.25 & 25.33 & 3.33 \\
         & AsymKV-10/40 & 42.30 & 41.00 & 12.64 & 28.65 & 5.10 \\
         & AsymKV-15/40 & 48.14 & 50.00 & 17.82 & 31.37 & 5.73 \\ \cmidrule{2-7}
         & AsymKV-40/0 & 86.70* & 67.50* & 41.90* & 46.92 & 8.78* \\ 
         & AsymKV-40/10 & 86.80* & 69.00* & 42.23* & 50.68 & 7.56* \\
         &AsymKV-40/15 & 87.39* & 69.00* & 42.45* & 50.25 & 8.58* \\
         \bottomrule
    \end{tabular}
    }
    \caption{Evaluation on LongBench tasks (*: {\ourmodel} achieves at least 90\% performance of floating-type models).}
    \label{tab:full-exp-long}
\end{table*}




\end{document}